%% file: main.tex
\documentclass[doublecolumn]{IEEEtran}

\usepackage{cite}
\usepackage{cuted}
\usepackage{caption}
\usepackage{comment}
\ifCLASSOPTIONcompsoc
\usepackage[caption=false, font=normalsize, labelfont=sf, textfont=sf]{subfig}
\else
\usepackage[caption=false, font=footnotesize]{subfig}

\ifCLASSINFOpdf
  \usepackage[pdftex]{graphicx}
\else
  \usepackage[dvips]{graphicx}
\fi
\usepackage{amsmath}
\usepackage{amssymb}
\usepackage{amsthm}
\makeatletter

\newcommand{\Rmnum}[1]{\expandafter\@slowromancap\romannumeral #1@}
\makeatother

\newtheorem{lemma}{Lemma}
\newtheorem{remark}{Remark}
\newtheorem{definition}{Definition}
\newtheorem{proposition}{Proposition}

\usepackage{subfig}

\usepackage{threeparttable}
\usepackage{makecell}

\usepackage{float}
\usepackage{bm}
\usepackage{algorithmic}
\usepackage{array}
\usepackage[table]{xcolor}
\usepackage{booktabs}
\usepackage{algorithm}
\usepackage{algorithmic}

\ifCLASSOPTIONcompsoc
  \usepackage[caption=false,font=normalsize,labelfont=sf,textfont=sf]{subfig}
\else
  \usepackage[caption=false,font=footnotesize]{subfig}
\fi

\ifCLASSOPTIONcaptionsoff
 \usepackage[nomarkers]{endfloat}
 \let\MYoriglatexcaption\caption
 \renewcommand{\caption}[2][\relax]{\MYoriglatexcaption[#2]{#2}}
\fi

\usepackage{url}

\begin{document}

\title{From Discrete to Continuous: Deep Fair Clustering With Transferable Representations}

\author{
Xiang~Zhang,~\IEEEmembership{Student~Member,~IEEE} 
              
\thanks{The authors are with the School of Information Science and Engineering, Southeast University, Nanjing 210096, China (e-mail: xiangzhang369@seu.edu.cn).
 
}

}

\maketitle

\input{Abstract/abstract}

\section{Introduction}
\label{sec:introduction}
\input{Introduction/introduction}

\section{Related Work}
\label{sec:relatedwork}
\input{Relatedwork/relatedwork}

\section{Problem Formulation and Motivations}
\label{sec:motivation}
\input{Motivation/motivation}

\section{The Proposed Deep Fair Clustering Method}
\label{sec:formulation}
\input{Formulation/formulation}

\section{Experiment}
\label{sec:experiment}
\input{Experiment/experiment}

\section{Conclusion}
\label{sec:conclusion}
\input{Conclusion/conclusion}


    \bibliographystyle{ieeetr}
\bibliography{abrv, references}

 \appendix

\input{Supplementary/supplementary}

\end{document}

%% file: Abstract/abstract.tex
\begin{abstract}
We consider the problem of deep fair clustering, which partitions data into clusters via the representations extracted by deep neural networks while hiding sensitive data attributes. To achieve fairness,  existing methods present a variety of fairness-related objective functions based on the group fairness criterion. However, these works typically assume that the sensitive attributes are discrete and do not work for continuous sensitive variables, such as the proportion of the female population in an area. Besides, the potential of the representations learned from clustering tasks to improve performance on other tasks is ignored by existing works. In light of these limitations, we propose a flexible deep fair clustering method that can handle discrete and continuous sensitive attributes simultaneously. Specifically, we design an information bottleneck style objective function to learn fair and clustering-friendly representations. Furthermore, we explore for the first time the transferability of the extracted representations to other downstream tasks. Unlike existing works, we impose fairness at the representation level, which could guarantee fairness for the transferred task regardless of clustering results. To verify the effectiveness of the proposed method, we perform extensive experiments on datasets with discrete and continuous sensitive attributes, demonstrating the advantage of our method in comparison with state-of-the-art methods.

\end{abstract}

%% file: Introduction/introduction.tex
Clustering is an unsupervised task that groups samples with common attributes and separates dissimilar samples.  It has been extensively used in numerous fields, e.g., image processing  \cite{lei2018superpixel}, remote sensing \cite{xie2018unsupervised}, and bioinformatics \cite{kiselev2019challenges}. Recently, many concerns have arisen regarding fairness when performing clustering algorithms. For example, in image clustering, the final results may be biased by attributes such as gender and race due to the imbalanced distribution of these sensitive attributes, even though the algorithm itself does not take these factors into account  \cite{chouldechova2018frontiers}. Unfair clustering could lead to discriminatory or even undesirable outcomes \cite{Zeng_2023_CVPR}, which brings a growing need for fair clustering methods that are unbiased by sensitive attributes.

Fair clustering aims to remove sensitive attributes from clustering results while preserving utility as much as possible. To achieve this goal, many studies attempt to incorporate fairness into traditional clustering algorithms. For example, pre-processing methods\cite{chierichetti2017fair,backurs2019scalable}  pack data points into fair subsets before conducting clustering, in-processing methods \cite{kleindessner2019guarantees, ziko2021variational} regard fairness as a constraint of traditional clustering optimization problems, and post-processing method \cite{bera2019fair} transforms the given clustering into a fair one by linear programming after executing traditional algorithms. Another line of works \cite{wang2019towards, zhang2021deep, li2020deep, Zeng_2023_CVPR} incorporates fairness into deep clustering methods. Specifically, these methods conduct clustering on the representations extracted from neural networks instead of raw data. To obtain fair clustering results, they carefully design clustering objects to ensure that sensitive subgroups are proportionally distributed across all clusters.

Although existing works have achieved impressive results, they still have some issues to address. First, these methods typically assume that sensitive attributes are discrete, such as gender and race. Continuous sensitive attributes are completely ignored, which, however, are common in the real world. To illustrate, let us take a real-world example here. We use the US Census dataset \cite{grari2019fairness}, which is an extraction of the 2015 American Community Survey, to group the tracts in America using $k$means. We consider a continuous variable \textemdash the proportion of the female population in a census tract\textemdash  as the sensitive attribute. As depicted in Fig \ref{fig-illustration}, we observe that the distributions of the sensitive variable in two different clusters exhibit different characteristics, indicating that the clustering results are biased by the sensitive attribute. Therefore, it is necessary to remove continuous sensitive information from clustering results just like removing discrete ones  \cite{wang2019towards, zhang2021deep, li2020deep, Zeng_2023_CVPR}. A naive method to adapt existing methods to continuous variables is to discretize the variables into ``categorical bins''. However, discretization could present threshold effects \cite{mary2019fairness}, and the thresholds may have no real sense.  
 \begin{figure}[t] 
    \centering
       \includegraphics[width=0.8\linewidth]{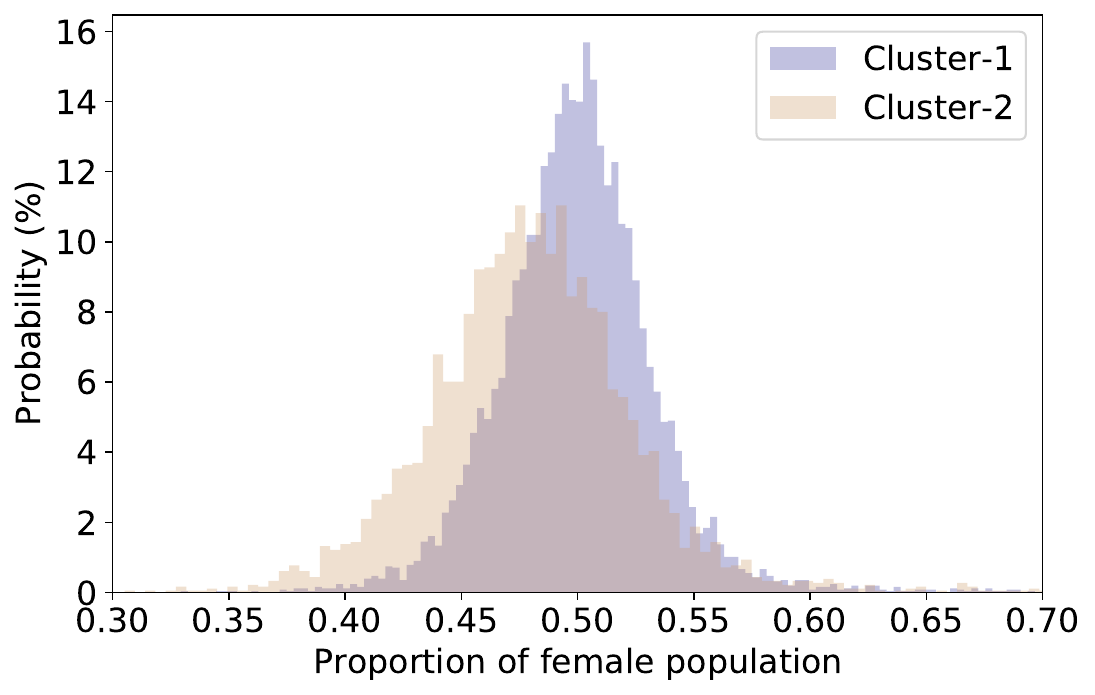}
    	\caption{We group the Census dataset into two clusters without considering fairness. The figure depicts the distributions of the proportion of the female population in the two clusters.    
}
    	\label{fig-illustration}
\end{figure}

On the other hand, in addition to clustering results, deep methods will produce a by-product, an encoder extracting representations from raw data.  A natural question is, \textit{can the encoder be used to learn representations for other downstream tasks?} A potential example is few-shot classification, where the sample size is too small to learn a strong classifier. The encoder trained for the clustering task is expected to generate representations with compact and separated cluster structures, which could improve classification performance. However, existing deep fair clustering methods \cite{wang2019towards, zhang2021deep, li2020deep, Zeng_2023_CVPR} ignore the transferability of the learned representations. Moreover, if the representations learned by existing models are transferred to other tasks, the fairness of their results cannot be guaranteed. The reason is that these methods obtain fair clustering by minimizing statistical dependence between clustering outputs and the sensitive attributes. However, imposing fairness on clustering results may not necessarily generate fair representations for unseen tasks. Therefore, we need fair representations that admit fair results on unseen tasks, even if the downstream tasks are not explicitly specified to be fair \cite{madras2018learning}.

To address the above issues, we propose a deep fair clustering method that can handle both discrete and continuous sensitive attributes. The main contributions of this study are summarized as follows.

\begin{itemize}
    \item We study fair clustering by considering both discrete and continuous sensitive attributes. To this end, we design an information bottleneck style clustering objective function to recast deep fair clustering as a mutual information optimization problem, where discrete and continuous attributes can be flexibly and uniformly handled.

    \item We investigate for the first time the potential of representations learned from deep clustering tasks to improve performance on other tasks. Unlike existing methods, we formulate the fairness of clustering by minimizing mutual information between sensitive attributes and latent representations. We theoretically demonstrate that this formulation can guarantee fairness when the representations are transferred to unseen tasks.

    \item We conduct extensive experiments on datasets with discrete and continuous properties to demonstrate the superiority of our approach over state-of-the-art methods. Furthermore, we also experimentally validate the transferability of the representations learned from deep fair clustering to few-shot fair classification.

\end{itemize}

%% file: Relatedwork/relatedwork.tex
\subsection{Fair Clustering}
In the literature, there have been abundant works \cite{ahmadian2019clustering, brubach2020pairwise, chen2019proportionally, davidson2020making,mahabadi2020individual} on fair clustering, which can be roughly divided into three categories, i.e., pre-processing, in-processing, and post-processing methods. Typically, pre-processing methods impose fairness constraints on the samples before performing traditional clustering methods. For example, \cite{chierichetti2017fair,backurs2019scalable} divides samples into several subsets termed Fairlets with fairness constraints before conducting clustering algorithms like $k$median. In-processing methods integrate fairness constraints into the clustering objective. One example is \cite{kleindessner2019guarantees}, which performs spectral clustering with a linear group fairness constraint. Besides, \cite{ziko2021variational} recasts the fairness constraint as a Kullback-Leibler (KL) term integrated into classic clustering methods. Finally, if we have obtained clustering results using traditional methods, post-processing methods like \cite{bera2019fair} can be employed to convert them into a fair one.

Different from traditional shallow models, deep clustering is an emerging method that utilizes powerful neural networks to improve clustering performance \cite{ghasedi2017deep,guo2017deep,
li2021contrastive,vincent2010stacked, yang2019deep}. Deep fair clustering incorporates fairness into deep clustering by carefully designing fairness-related objective functions. For example, \cite{wang2019towards} proposes a method that can handle multi-state variables, \cite{zhang2021deep} uses pseudo cluster assignments to facilitate model optimization, and \cite{li2020deep} employs an adversarial training framework to achieve fairness. Based on \cite{li2020deep}, \cite{chhabra2022robust} considers fair attacks for clustering, which is inconsistent with the purpose of this study. The most recent work \cite{Zeng_2023_CVPR} proposes an interpretable model via maximizing and minimizing mutual information. Our method differs from existing works in two ways. First, in addition to discrete sensitive attributes, our approach also considers the unexplored case where sensitive attributes are continuous. Several works \cite{mary2019fairness, grari2019fairness, lee2022maximal,chen2022scalable} measure fairness w.r.t. continuous sensitive variables using various statistical tools such as  R\'{e}nyi maximal correlation, HGR maximal correlation, and slice mutual information. Instead, our method employs mutual information to estimate fairness. Besides, these methods only focus on supervised tasks, while our method focuses on unsupervised clustering tasks. Second, compared with current deep methods  \cite{wang2019towards, zhang2021deep, li2020deep, Zeng_2023_CVPR}, our method can produce fair and transferable representations.

\subsection{Fair Representation Learning}
Fair representation learning (FRL) aims to learn latent representations that are unbiased by sensitive attributes \cite{zemel2013learning}. Several FRL methods learn representations for specific downstream tasks, such as classification with labels \cite{zhu2021learning, shui2022fair,gordaliza2019obtaining,8438994}. On the other hand, many works focus on learning fair representations without relying on labels from downstream tasks \cite{mehrabi2021survey}. These works achieve fairness by reducing mutual information between sensitive attributes and representations \cite{gupta2021controllable, song2019learning, xie2017controllable, madras2018learning, mary2019fairness}. The representation learned by our method is a by-product of a specific unsupervised downstream task, deep fair clustering, which has not been explored before in FRL.  We show that the additional clustering objectives increase the utility and transferability of the representations, especially for classification tasks.

%% file: Motivation/motivation.tex
\subsection{Problem Formulation}
Given a set of $N$ data points $\{X_1,...,X_N\}$, each point $X_i$ has a sensitive attribute $G_i$ whose values could be discrete, e.g., gender and race, or continuous, e.g., the proportion of a certain gender in the population. Fair clustering aims to find an assignment function $\kappa: X\to C\in [K]$ \footnote{$[K] = \{1,...,K\}$} that groups data points into clusters $C$ while simultaneously guaranteeing that the clustering is unbiased by the sensitive variable $G$. To achieve this goal, deep methods employ neural networks to extract latent representations $Z$ from raw data and perform clustering on the representations in an end-to-end manner.

\subsection{Unifying Existing Methods Via Mutual Information}
The key to deep fair clustering is how to incorporate fairness into the clustering objective function. To this end, existing methods design various fairness-related objective functions using different fairness criteria, which are listed in Table \ref{table-fairness-metric}. We will show that all these criteria can be interpreted uniformly through mutual information.

\begin{table}[htbp]
\renewcommand{\arraystretch}{1.5}
	\centering
  
	\begin{threeparttable}
        \tabcolsep = 0.2em
	\caption{Fairness criteria of existing deep methods }
	{\footnotesize
	\begin{tabular}{c|c|c}
	\Xhline{1.05pt}
	
 Index  &  Methods & Fairness criteria  \\
     
    \hline

    1 & \cite{Zeng_2023_CVPR}
     &$    p(C =k, G =t) = p(C =k) p(G  =t), \; \forall t\in[T], \forall k\in[K]$ \\  	
   2 &\cite{li2020deep}
     & $ \mathbb{E}_{p(X)}[ G | C=k] = \mathbb{E}_{p(X)} [G],\; \forall k\in[K]$ \\  		
  3  & \cite{zhang2021deep}
     & $ p(G=t|C=k) =  p(G=t) ,\; \forall t \in[T]$\\  		
   4   & \cite{wang2019towards}
     & $ p(G=t_1|C=k) =  p(G=t_2|C=k),\; \forall t_1, t_2\in[T]$\\  

	\Xhline{1.2pt}

	\end{tabular} 
	}
 \label{table-fairness-metric}
 \end{threeparttable}
\end{table}

\begin{proposition}
The fairness criteria in Table \ref{table-fairness-metric}  are equivalent to finding a clustering that satisfies $I(C;G) = 0$.

\label{prop-1}
\end{proposition}

\begin{proof}
First, if  criterion 1 in Table \ref{table-fairness-metric} is satisfied, it is not difficult to check that $\frac{ p(C=k, G=t )}{p(C=k) p(G =t )} = 1,\forall t\in[T], k\in[K]$, leading to $I(G;C) = 0$ according to the definition of mutual information. Second, the fairness criteria of 2-3 all require that the distribution of $G$ in each cluster is close to the whole dataset. This indicates that $\frac{ p( G=t|C = k )}{p(G=t)} = \frac{ p(C=k, G=t )}{p(C=k) p(G=t  )} = 1, \forall t\in[T],k\in[K]$, leading to $I(G;C) = 0$. Third, if metric 4 is satisfied, we have $p(G = t|C = k) = \frac{1}{T}$. On the other hand, $p(G = t) = \sum_{k=1}^K p(G = t,C = k) = \sum_{k=1}^K p(G = t|C = k)p(C = k) = \frac{1}{T}\sum_{k=1}^K p(C = k) = \frac{1}{T}$. Naturally, we have $p(G = t|C = k)= p(G = t)$, meaning that $I(G;C) = 0$.  
In summary, all fairness criteria in Table \ref{table-fairness-metric} indicate $I(C;G) = 0$.
\end{proof}

\subsection{Limitations and Motivations}
As observed from Table \ref{table-fairness-metric} and Proposition \ref{prop-1},  existing deep fair clustering methods have two main characteristics. (\romannumeral1) The sensitive attributes are assumed to be discrete, i.e., $G \in[T]$. (\romannumeral2) The fairness is achieved by directly minimizing statistical independence between sensitive attributes $G$ and the final clustering $C$, i.e., $I(C;G) = 0$. However, in the real world, we usually encounter scenarios where sensitive attributes are continuous, and existing works cannot handle such scenarios. Furthermore, the representations learned from clustering tasks could facilitate other tasks, such as classification. These works completely ignore the potential transferability of the representations learned from clustering tasks. Besides, minimizing $I(C;G)$ does not necessarily lead to representations that are statistically independent of sensitive attributes. When the representations are used for other downstream tasks, fairness may be no longer guaranteed.

The two observations motivate us to explore a unified deep fair clustering method that can (\romannumeral1) handle both \textbf{discrete} and \textbf{continuous} sensitive attributes, and (\romannumeral2) produce \textbf{fair} and \textbf{transferable} representations for other tasks.

%% file: Formulation/formulation.tex
In this section, we reformulate the fair clustering problem by considering both discrete and continuous sensitive attributes. Then, we propose a deep fair clustering method via the lens of FRL, where fair and clustering-friendly representations for downstream tasks can be obtained from a unified framework.

\subsection{Fair Clustering Reformulation}
As concluded in Proposition \ref{prop-1}, existing deep fair clustering methods aim to seek clustering that does not share mutual information with discrete sensitive attributes. Inspired by the proposition, it is natural to generalize existing fairness criteria in clustering tasks to continuous sensitive attributes from the perspective of mutual information as follows.
\begin{definition}
For $C\in [K]$ and  $G$ being discrete or continuous, the clustering is absolutely fair if $I(C;G) =0$.
   \label{fairness-unified} 
\end{definition}
This definition of fairness is also known as the $\mathrm{Independence}$ criterion in fair classification tasks \cite{barocas2023fairness, mehrabi2021survey}. Here, we use it to uniformly define the fairness of clustering results w.r.t. both discrete and continuous sensitive attributes.

However, in this study, we minimize the mutual information $I(Z;G)$ instead of directly minimizing $I(C;G)$ to obtain fair clustering. According to the data processing inequality, $I(C;G)\leq I(Z;G)$. Thus, $ I(Z;G)$ is an upper bound of $I(C;G)$, and minimizing $I(Z;G)$ leads to decreasing  $I(C;G)$, which produces fair clustering results. The merit of minimizing $I(Z;G)$ instead of $I(C;G)$ is that fairness can be guaranteed when the representations are transferred to other downstream tasks. To illustrate, let us assume that the downstream task is classification $\lambda : Z \to \widehat{Y}$, where $\widehat{Y}$ is the prediction of $\lambda$. Furthermore, we follow \cite{gupta2021controllable} and assume $G$ is discrete since the fairness of classification results w.r.t. discrete sensitive attributes can be explicitly measured via demographic parity (DP) \cite{dwork2012fairness}, i.e., $\Delta_{DP}(\lambda, G) = \left|p( \widehat{Y} = 1 |G = 1) - p( \widehat{Y} = 1 |G = 0) \right|$. Note that DP is built on binary classification with binary-sensitive attributes. In this study, we consider a more general metric called generalized population parity (GDP), where categorical labels and sensitive values are not limited to binary values.
\begin{definition}
For $L,T\geq 2$, GDP is defined as 
\begin{align}
   &\Delta_{GDP}(\lambda, G) \notag\\
   = &\max_{l, t\neq t^{\prime}}\left|p( \widehat{Y} = l |G = t) - p( \widehat{Y} = l |G = t^{\prime}) \right|, l\in [L], t, t^{\prime}\in [T]
\end{align}
    \label{definition-SP}
\end{definition}
By definition, if the classification results are completely insensitive to $G$, we have $\Delta_{GDP}(\lambda, G) = 0$. Then, we can obtain the following proposition.
\begin{proposition}
 For any $\lambda$ acts on $Z$ to produce predictions $\widehat{Y}$, we have 
    \begin{align}
        h(\eta \Delta_{GDP} (\lambda, G)) \leq I(Z;G),
   \label{eq-formulation-6}
\end{align}
where $h(\cdot)$  is a strictly increasing non-negative convex function, and $\eta = \min_{t\in[T]} p(G = t)$.
    \label{proposition-SP}
\end{proposition}
\begin{proof}
    The proof is in supplementary materials.
\end{proof}
The proposition is a generalized version of Theorem 2 in \cite{gupta2021controllable}, which tells us that $I(Z;G)$ upper bounds $\Delta_{GDP} (\lambda, G)$ since $ h(\eta \Delta_{GDP} (\lambda, G))$ is a strictly increasing convex function in $\Delta_{GDP}(\lambda, G)$. That is, any classification algorithm $\lambda$ acting on $Z$ to make a decision $\widehat{Y}$ will have bounded $\Delta_{DP}(\lambda, G)$ if $I(Z;G)$ is bounded. Therefore, if we transfer the representations learned by minimizing $I(Z;G)$ to classification tasks, the fairness of the classification results can be guaranteed regardless of the clustering outputs. 

\begin{remark}
When $Z$ is continuous, the metrics for evaluating the fairness of classification results have not been thoroughly studied. A direct proxy is the mutual information $I(\widehat{Y}, G)$. Using this proxy, we have $I(\widehat{Y}, G) \leq I(Z;G)$ due to the data processing inequality. Thus, minimizing $I(Z;G)$ also leads to fair classification results for continuous sensitive attributes.
    \label{remark-prop-2}
\end{remark}

\begin{figure*}[t] 
    \centering
       \includegraphics[width=0.8\linewidth]{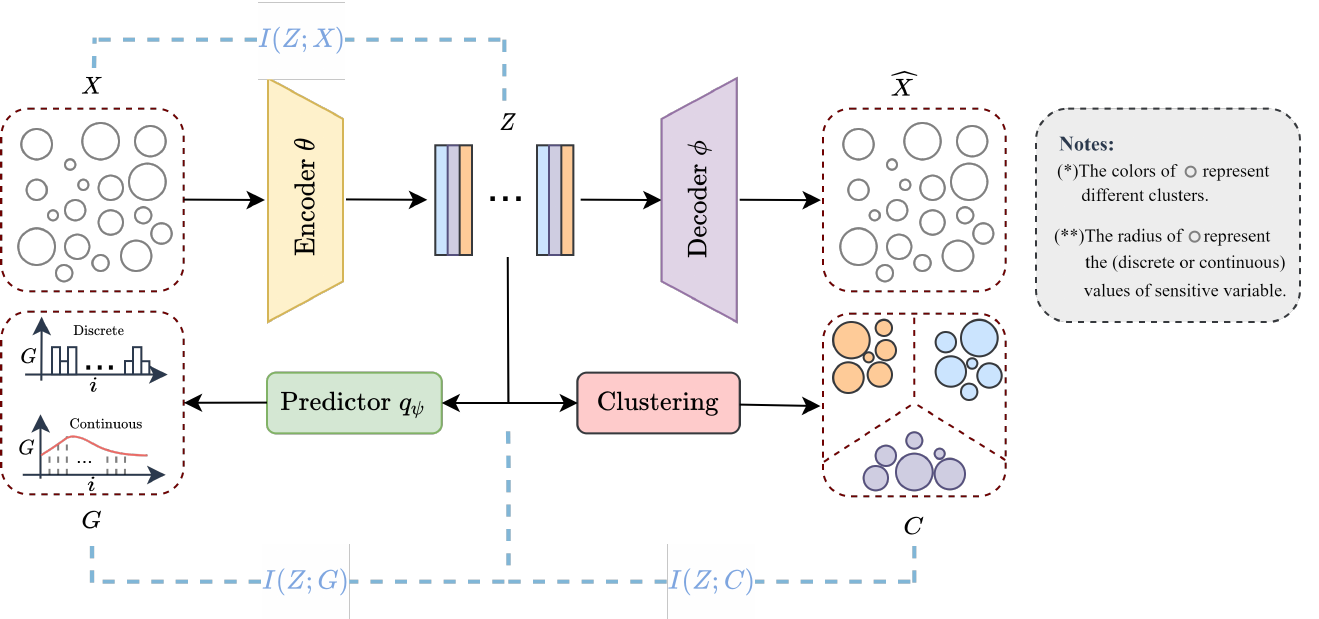}
    	\caption{The illustration of the proposed method.}
    	\label{fig-flowchart}
    	\vspace{-1em}
\end{figure*}

\subsection{Objective Functions}

\subsubsection{Framework overview}
We build our deep fair clustering framework based on the above fair clustering reformulation, as shown in Fig.\ref{fig-flowchart}. Consistent with existing deep methods, we leverage an auto-encoder framework that extracts latent representations $Z$ from raw data $X$. Then, clustering is performed on the representation space. To improve fair clustering performance, the representations are expected to be (\romannumeral1) clustering-favorable, (\romannumeral2) preserve useful information from original data $X$, and (\romannumeral3) be unbiased by the sensitive attributes $G$. Based on these requirements, we propose the following ``information bottleneck" style objective function at the presentation level
\begin{equation}
\max \;   I (Z;X) + \alpha I(Z;C) - \beta I(Z;G),
\label{objective function}
\end{equation}
where $\alpha$ and $\beta$ are trade-off parameters. Specifically, by maximizing mutual information $I(Z; C)$, we can obtain clustering-friendly representations. Besides, maximizing $I(Z; X)$ ensures the learned representations will not lose too much original data information. Finally, according to the above analysis, we remove sensitive information by minimizing $I(Z; G)$ instead of $I(Z; C)$ to produce transferable and fair representations. Next, we present the details of the three optimization objectives.

\subsubsection{$I(Z;C)$}
It is not difficult to obtain that 
\begin{align}
& I(Z;C)  \notag \\
=& H(C) - H(C|Z) \notag \\
=& -\mathbb{E}_{p(C)} \log p(C) + \mathbb{E}_{p(C,Z)} \log p(C|Z)\notag\\
= &{- \sum_{k=1}^{K} p(C=k)\log p(C=k)}\label{H1} \\
       &{+ \frac{1}{N}\sum_{i=1}^{N}\sum_{k=1}^{K} p(C=k|Z = Z_i)\log p(C=k|Z = Z_i)}\label{H2}\\
        := &-\mathcal{L}_c \notag,
\end{align}
where $Z_i = f_{\theta}(X_i)$, $p(C=k) = \frac{1}{N}\sum_{i=1}^N p(C=k|Z = Z_i)$, and $p(C=k|Z = Z_i)$ is the assignment probability of the $i$-th representation being grouped into the $k$-th cluster. In the literature, the assignment probability can be calculated in many ways, such as the Student t-distribution \cite{li2020deep}. In this study, we employ the vanilla $k$-means to compute the probability \cite{Zeng_2023_CVPR}
\begin{align}
p(C=k|Z = Z_i) = \frac{\exp(\cos(Z_i, U_k))/\tau}{\sum_{j=1}^{K} \exp(\cos(Z_i, U_j))/\tau}
\label{equ-formulation-prob}
\end{align}
where $U_k, k=1,...,K$, are the cluster centers, and $\tau = 0.1$  is the temperature to control the softness. Intuitively, maximizing \eqref{H1} punishes too large or too small clusters to avoid trivial solutions where all samples are assigned to the same clusters. Moreover, the maximization of \eqref{H2} encourages each representation to move toward its nearest cluster center and away from others, bringing compact clustering. In summary, maximizing $I(Z;C)$ can lead to clustering-favorable representations.

\subsubsection{$I(Z;X)$}
Since clustering is an unsupervised task, we maximize $I(Z;X)$ as a  self-supervised objective to learn informative representations. Due to the data processing inequality, we have $I(Z;X) \geq I(\widehat{X};X)$, where $\widehat{X} = g_\phi(Z)$ is the reconstructed data using $Z$ via a decoder $g_\phi(\cdot)$, and $\phi$ is the parameters of the decoder. We can maximize $I(\widehat{X};X)$ instead of $I(Z;X)$ since $I(\widehat{X};X)$ is the lower bound of $I(Z;X)$. Similarly, by the definition of mutual information, we have 
\begin{align}
I(\widehat{X};X) =  H(X) - H(X|\widehat{X}), 
\end{align}
where $ H(X)$ is a constant because the data are given in prior, and $H(X|\widehat{X})$ is calculated by 
\begin{align}
H(X|\widehat{X}) = - \mathbb{E}_{p(\widehat{X}, X)} \left[\log p(X|\widehat{X})  \right] = -\frac{1}{N}\sum_{i=1}^N  \log p(X_i|\widehat{X}_i).
\label{equ-formulation-1}
\end{align}
Here, we assume that $p(X_i|\widehat{X}_i)$ follows $\mathcal{N}(\widehat{X}_i, \widetilde{\sigma}^2I)$, where $I$ is an identity matrix. By removing constants, we obtain that maximizing $I(\widehat{X};X)$ is equivalent to minimizing 
\begin{align}
\mathcal{L}_r = \frac{1}{N} \sum_{i=1}^N \left \lVert \widehat{X}_i -  X_i\right\rVert_2^2.
\label{equ-formulation-2}
\end{align}

\begin{remark}
Note that minimizing $\mathcal{L}_c$ and $\mathcal{L}_r$ is similar to that in \cite{Zeng_2023_CVPR}, which can be also found in some deep auto-encoder clustering methods \cite{ren2022deep}. However, our method differs from the existing methods in that we derive the objective functions at the representation level, while others are at the sample level.
    \label{remark}
\end{remark}

\subsubsection{$I(Z;G)$}
It is difficult to calculate $I(Z;G)$ exactly since $Z$ is continuous. Here, we construct an upper bound of $I(Z;G)$  and minimize the upper bound instead. 
\begin{proposition}
     For mutual information $I(Z;G)$, we have 
\begin{align}
    I(Z;G) \leq \mathbb{E}_{p(Z,G)}\log p(G|Z)-\mathbb{E}_{p(Z)p(G)}\log p(G|Z),
    \label{equ-formulation-2-1}
\end{align}
    \label{proposition-upper-bound}
\end{proposition}
The upper bound in \eqref{equ-formulation-2-1} is based on the  Contrastive Log-ratio Upper Bound (CLUB) in \cite{cheng2020club}. However, in our problem, $p(G|Z)$ is usually unknown. Thus, we use a neural network to learn a variational approximation $q_{\psi}(G|Z)$ for $p(G|Z)$, where $\psi$ is the parameters of the neural network. Specifically, the approximation is obtained by solving the following likelihood
\begin{align}
\max_{\psi}\; \mathbb{E}_{p(G,Z)}   \log q_{\psi}(G|Z).
    \label{equ-formulation-3}
\end{align}
We consider two cases where $G$ is continuous or discrete. When $G$ is \textbf{discrete}, we assume $q_{\psi}(\cdot)$ obeys categorical distribution, i.e., $q_{\psi}(G|Z) = \prod_{t=1}^T ([\xi(Z)]_t)^{\mathbb{I}(G = t)}$, where $\mathbb{I}(\cdot)$ is an indicator function. Besides, $[\xi(Z)]_t$ is the probability of $Z$ belonging to $t$-th sensitive category predicted by a classifier. Thus, solving \eqref{equ-formulation-3} is equivalent to minimizing the following cross-entropy loss
\begin{align}
\mathcal{L}_d = -\frac{1}{N}\sum_{i=1}^N \sum_{t=1}^{T} \mathbb{I}(G_i = t) \log[\xi(Z_i)]_t.
    \label{equ-formulation-3-1}
\end{align}
When $G$ is \textbf{continuous}, we assume that $q_{\psi}(\cdot)$ is the Gaussian distribution, i.e., $q_{\psi}(G|Z) = \mathcal{N}(\mu(Z), \sigma^2(Z)I )$, where $\mu(Z)$ and $ \sigma^2(Z)$ are inferred by neural networks. This assumption is common in variational methods \cite{doersch2016tutorial}. Thus, solving \eqref{equ-formulation-3} is equivalent to minimizing
\begin{align}
\mathcal{L}_q =  \frac{1}{N}\sum_{i=1}^N \log\sigma^2(Z_i) + \frac{(G_i - \mu(Z_i))^2}{\sigma^2(Z_i)}.
    \label{equ-formulation-3-2}
\end{align}

After obtaining $q_{\psi}(G|Z)$, with samples $\{(Z_i, G_i)\}_{i=1}^N$, we estimate the upper bound of $I(Z;G)$ via minimizing
\begin{align}
\mathcal{L}_s = \frac{1}{N}\sum_{i=1}^N \left( \log q_{\psi}(G_i|Z_i) -  \frac{1}{N} \sum_{j=1}^N \log q_{\psi}(G_j|Z_i)
\right).
    \label{equ-formulation-4}
\end{align}

\subsection{The Complete Procedure}
At the beginning of each epoch, we extract latent representations of the whole dataset via $Z_i = f_{\theta}(X_i)$ and compute cluster centers $\{U_k\}_{k=1}^K$ by applying $k$means on the extracted representations. Then, for each minibatch, we train the predictor $\psi$ by minimizing \eqref{equ-formulation-3-1} (discrete sensitive attributes) or \eqref{equ-formulation-3-2} (continuous sensitive attributes). After obtaining $q_{\psi}$, we fix $\psi$ and update encoder $\theta$ and decoder $\phi$ by minimizing the following objective function
\begin{align}
\mathcal{L} = \mathcal{L}_r + \alpha \mathcal{L}_c + \beta \mathcal{L}_s.
    \label{equ-formulation-5}
\end{align}
Repeating the above process until convergence, we finally obtain the encoder $f_{\theta}$ to generate fair representations that can be assigned to the clusters corresponding to the nearest cluster centers. Following \cite{Zeng_2023_CVPR}, we leverage a warm-up strategy to facilitate convergence, in which only $\mathcal{L}_c$ are minimized. The complete procedure is placed in Algorithm \ref{alg:1}.

At first glance, the entire procedure follows the paradigm of adversarial training, in which we learn a discriminator $q_{\psi}$ to predict sensitive attributes $G$ from $Z$, including discrete and continuous sensitive attributes. Note that \cite{li2020deep} also leverages the adversarial training technique to design fairness loss. Our method differs \cite{li2020deep} in three-fold. First, our fairness-adversarial loss is theoretically derived from minimizing $I(Z;G)$, while that of \cite{li2020deep} is designed heuristically. Second, our loss is at the representation level while that of \cite{li2020deep} is based on clustering results. Third, we consider both discrete and continuous cases while \cite{li2020deep} only considers categorical sensitive variables.



\begin{algorithm}[t] 
\caption{The complete procedure of the proposed method} 
\begin{algorithmic}[1] 
\REQUIRE  Samples $X_1,...,X_N$ with sensitive attributes $G_1,...,G_N$,  $\alpha$, $\beta$, $epoch_{\max}$ and $epoch_{\mathrm{warm}}$

\ENSURE  The fair clustering $C_1,...,C_N$ and representations $Z_1,...,Z_N$.

\FOR{$epoch<epoch_{\max}$  }

\STATE Extract latent representations via $Z_i = f_{\theta}(X_i), i = 1...,N$
\STATE Obtain the cluster centers $U$ and the clustering $C$ via $k$means on $Z$
\FOR{ each minibatch}
\IF{$epoch < epoch_{\mathrm{warm}}$}
   \STATE Perform warm-up steps
\ELSE
   \STATE Update predictor $\psi$ by minimizing  \eqref{equ-formulation-3-1}  or \eqref{equ-formulation-3-2} 
    \STATE Fix $\psi$ and compute objective function via \eqref{equ-formulation-5}
\ENDIF
\STATE Update encoder parameters $\theta$ and decoder parameters $\phi$ to minimize 
 \eqref{equ-formulation-5} via stochastic gradient descent (SGD).
\ENDFOR
\ENDFOR

\end{algorithmic}
\label{alg:1} 
\end{algorithm}

%% file: Experiment/experiment.tex
\begin{table*}[t]
\renewcommand{\arraystretch}{1.3}
	\centering
	\begin{threeparttable}
        \tabcolsep = 0.1em
	\caption{Clustering results on the datasets of discrete sensitive attributes.}
	{\footnotesize
	\begin{tabular}
 {c|ccccc|ccccc|ccccc|ccccc|ccccc}
	\Xhline{1.05pt}

    	 & \multicolumn{5}{c|}{MNIST-USPS} & \multicolumn{5}{c|}{Color Reverse MNIST}& \multicolumn{5}{c|}{HAR} & \multicolumn{5}{c|}{Offce-31}& \multicolumn{5}{c}{MTFL}\\

	\cline{2-26}

     & $\mathrm{ACC} $ & $\mathrm{NMI} $  & $\mathrm{Bal}$  & $\mathrm{MNCE}$ 
     & $F_m$ & $\mathrm{ACC} $ & $\mathrm{NMI} $  & $\mathrm{Bal}$  & $\mathrm{MNCE}$ 
     & $F_m$& $\mathrm{ACC} $ & $\mathrm{NMI} $  & $\mathrm{Bal}$  & $\mathrm{MNCE}$ 
     & $F_m$& $\mathrm{ACC} $ & $\mathrm{NMI} $  & $\mathrm{Bal}$  & $\mathrm{MNCE}$ 
     & $F_m$& $\mathrm{ACC} $ & $\mathrm{NMI} $  & $\mathrm{Bal}$  & $\mathrm{MNCE}$ 
     & $F_m$  \\
     
    \hline

    AE 
    &76.3 &71.8 &0.0 &0.0 &0.0
    &41.0 &52.8 &0.0 &0.0 &0.0
    &66.3 &60.7 &0.0 &86.9 &71.5
    &63.8 &66.8 &0.0 &0.0 &0.0
    &67.2 &16.0 &67.8 & 97.3 &27.5\\

    DEC  
    &60.0 &59.4 &0.0 &0.0 &0.0
    &40.7 &38.2 &0.0 &0.0 &0.0
    &57.1 &65.5 &0.0 &93.7 &77.1 
    &63.3 &68.6 &0.0 &0.0 &0.0
    &56.7 &0.6 &78.0 &98.9 &1.1\\

     DAC 
    &76.3 &69.9 &0.0 &0.0 &0.0
    &31.4 &27.1 &0.0 &0.0 &0.0
    &38.2 &31.5 &0.0 &32.4 &31.9 
    &14.0 &25.2 &0.0 &0.0 &0.0
    &58.9 &1.4 &81.5 &87.9 &2.7\\

     CIGAN 
    &38.3 &35.7 &0.1 &1.9 &3.6
    &20.1 &9.1 & 2.2 &14.9 & 11.3
    &52.7 & 44.3 & 0.4 &0.0 &0.0
    &52.2 & 54.9 &0.0 &0.0 &0.0
    &\underline{72.9} &12.6 &79.1 &99.0 &22.4\\

     ScFC 
    &14.2 & 1.3 &\textbf{11.2} &\textbf{95.0} &2.6
    &51.3 &49.1 &\textbf{100.0} &\textbf{100.0} &65.8
    &\textemdash &\textemdash &\textemdash &\textemdash &\textemdash 
    & 38.0 &60.7 &\textbf{26.7} &\textbf{97.7} &74.9
    &52.1 &15.1 &\textbf{100.0} &\textbf{100.0} &26.3\\

    SpFC 
    &20.1 &15.5 &0.0 &0.0 &0.0
    &11.0 & 2.1 &0.0 &0.0 &0.0
    &19.0 &0.4 &0.0 &0.0 &0.0 
    &9.3 & 11.4 &0.0 &0.0 &0.0
    & 65.5 &0.1 & 75.0 &98.5 &0.2\\

    VFC 
    & 58.1 & 55.2 &0.0 &0.0 &0.0
    & 38.1 & 42.7 &0.0 &0.0 &0.0
    & 62.6 & 66.2 & 25.6 & 98.7 & 79.3 
    & 65.2 &69.7 & 20.3 &86.0 &77.0
    &68.8 &8.4 &88.9 &99.8&  15.6\\

    FAlg 
    & 58.4 &53.8 &9.5 & 85.8 &66.1
    &26.9 & 14.3 &66.6 &97.1 &24.9
    & 56.6 &58.6 &\underline{43.2} &99.2 &73.7
    & 67.1 & 70.7 & 20.4& 86.4 & \underline{77.8}
    &63.2 &16.7 &60.1& 96.3 &28.5\\

    DFC 
    & 85.7 & 83.4 & 6.7 & 68.2 & 75.0
    &49.9 & 68.9 & 80.0 &99.1 & 81.3
    &\textemdash &\textemdash &\textemdash &\textemdash &\textemdash
    &\underline{69.0}  &\underline{70.9} &11.9 & 64.2 &67.4
    & 72.8 &17.6 &\underline{97.4} &\underline{99.9} &30.0\\

    FCMI
    &\underline{96.7} &\underline{91.8} &\underline{10.7} &\underline{94.5} &\textbf{92.0}
    & \underline{88.4} & \underline{86.4} &99.5 &\underline{99.9} &\underline{92.7}
    &\textbf{88.2} &\textbf{80.7} &40.7 &\underline{99.3} & \textbf{89.0}
    &\textbf{70.0} &\textbf{71.2} &\underline{22.6} &\underline{90.6} &\textbf{79.7}
    &70.2 &\underline{19.1} &90.4 &99.8 &\underline{32.0}\\

    Ours
    &\textbf{97.0} &\textbf{92.4} &10.5 &91.5 &\underline{91.7} 
    &\textbf{97.0} &\textbf{92.3} &\underline{99.7} &\textbf{100.0} &\textbf{96.0}
    &\underline{82.5} &\underline{75.8} &\textbf{45.2} &\textbf{99.5} &\underline{85.9} 
    &64.0 &67.6 &19.4 &84.1 &74.9
    &\textbf{85.8} &\textbf{27.7} &89.1 &\underline{99.9} &\textbf{43.4}\\

 \Xhline{1.05pt}
	\end{tabular} 
	}
		\begin{tablenotes}
			\footnotesize
			\item  The best and the second best results are marked in bold and underline, respectively. 
        \end{tablenotes}

          \label{table-discrete}
 \end{threeparttable}
     	\vspace{-1em}
\end{table*}

\begin{figure*}[t] 
    \centering
\subfloat[]{
    \includegraphics[width=0.49\linewidth]{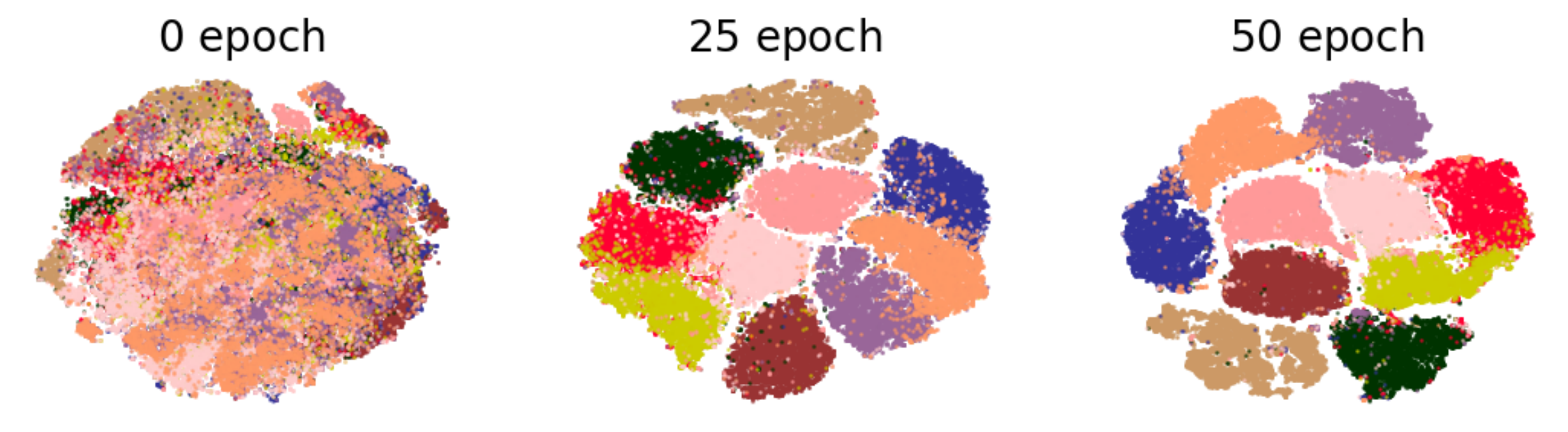}
}
\subfloat[]{
    \centering
    \includegraphics[width=0.49\linewidth]{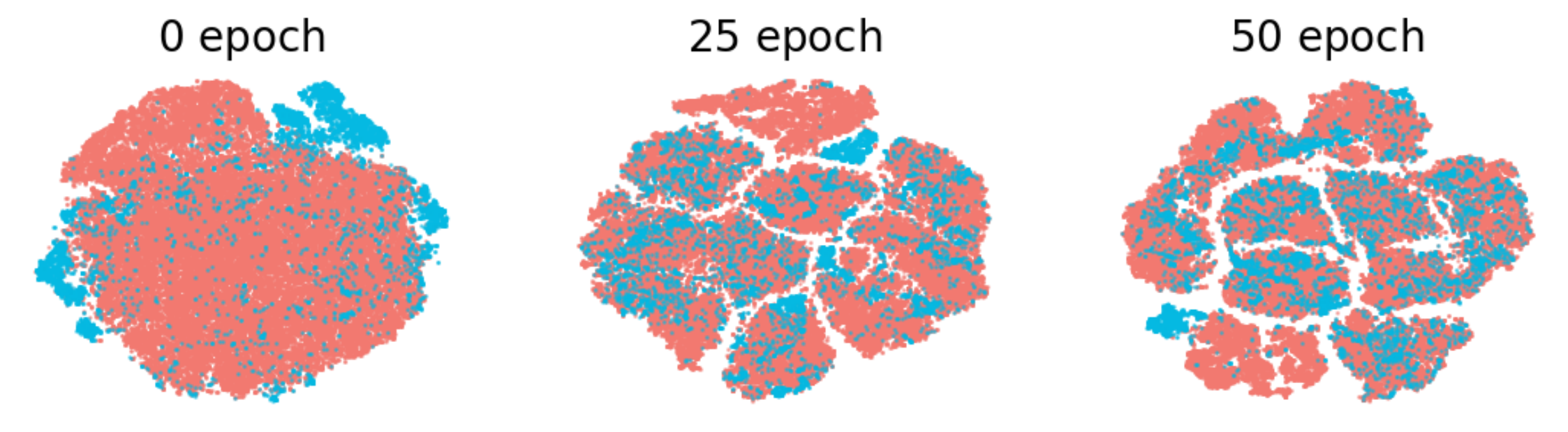}
}

    	\caption{The results of visualization of the representations on MNIST-USPS datasets. (a) The left three figures are colored by classes, and (b) the right three figures are colored by groups.
     }
    	\label{fig-visual}
\end{figure*}

In this section, we evaluate the proposed method by answering the following questions:
\begin{itemize}
\item[(Q1)] Compared to existing methods, can our method achieve competitive results on datasets with \textbf{discrete} sensitive attributes? 
\item[(Q2)] Can our model effectively obtain fair clustering on datasets with \textbf{continuous} sensitive attributes?
\item[(Q3)] How do the learned representations perform when transferred to other tasks?
\end{itemize}

\subsection{Experimental Setups}
\textbf{Datasets: } For Q1, we employ five commonly used datasets with discrete sensitive attributes in fair clustering tasks, i.e., MNIST-USPS \footnote{http://yann.lecun.com/exdb/mnist, https://www.kaggle.com/bistaumanga/usps-dataset}, ReverseMNIST, HAR \cite{anguita2013public}, Offce-31 \cite{mukherjee2019clustergan}, and MTFL \cite{zhang2014facial}. For Q2, the US Census  dataset \cite{grari2019fairness, mary2019fairness} and Communities and Crimes  (C\&C) dataset \cite{grari2019fairness} are employed. The US Census dataset is an extraction of the 2015 American community survey, where we use the proportion of the female population in a census tract as the sensitive attribute and whether the average income in this tract is above 50,000 USD as the cluster label. The Crime dataset includes 128 attributes for 1,994 instances from communities
in the US. We use the population proportion of a certain race as the sensitive attribute. Whether the number of violent crimes per population is greater than 0.15 is used as the cluster label. More details of the used datasets can be found in the appendix.

\textbf{Implementation Details: }
Following the previous works \cite{Zeng_2023_CVPR, li2020deep}, we utilize a convolutional auto-encoder for
MNIST-USPS and Reverse MNIST datasets, and a fully-connected auto-encoder for other datasets. For fair comparisons, we employ group-wise decoders as in \cite{Zeng_2023_CVPR} for datasets with discrete sensitive attributes. The ResNet50 \cite{he2016deep}  is used to extract features of MTFL and Office-31 as the inputs. To estimate mutual information of $I(Z; G)$, we employ a three-layer fully-connected neural network with the hidden layer size of $16$ as the predictor. For all experiments, we fix $\alpha = 0.04$ and $\beta = 0.18$. We train our model for 300 epochs using the Adam optimizer with an initial learning rate of $10^{-4}$ for all datasets. Consistent with \cite{Zeng_2023_CVPR}, the first  20 epochs are warm-up steps.  All experiments are conducted on an RTX 2080Ti GPU. The source code is available at https://github.com/kalman36912/DFC\_Dis\_Con/tree/master.

\textbf{Baselines: } For Q1, we compare our model with all baselines in \cite{Zeng_2023_CVPR}, which includes classic clustering methods and fair clustering methods. We use four  classic methods, i.e., auto-encoder+$k$means \cite{vincent2010stacked}, DEC \cite{xie2016unsupervised}, DAC \cite{chang2017deep} and ClGAN \cite{mukherjee2019clustergan}. For fair clustering methods, four shallow methods are employed, i.e., ScFC \cite{backurs2019scalable}, SpFC \cite{kleindessner2019guarantees}, VFC \cite{ziko2021variational} and  FAlg  \cite{bera2019fair}. In addition, two deep fair clustering methods are included, i.e.,  DFC \cite{li2020deep} and FCMI \cite{Zeng_2023_CVPR}. Consistent with \cite{Zeng_2023_CVPR}, we ignore DFDC \cite{zhang2021deep} and Towards \cite{wang2019towards} since their codes are not available. Furthermore, ScFC \cite{backurs2019scalable} and DFC  \cite{li2020deep} only support the case of two sensitive groups. Therefore, we omit the results of both methods on the HAR dataset, which has 30 sensitive groups. For Q2, since existing fair clustering methods do not work for continuous sensitive attributes, we binarize continuous variables into discrete variables to adapt to existing methods. We only employ deep clustering methods, DFC \cite{li2020deep} and FCMI \cite{Zeng_2023_CVPR}, as they can learn latent representations. For clarity, the baselines of Q3 will be presented in context.

\textbf{Evaluation Metrics: } For Q1, we utilize all evaluation metrics in \cite{Zeng_2023_CVPR}, i.e., $ \mathrm{ACC}, \mathrm{NMI}, \mathrm{ Bal}, \mathrm{MNCE}$ and $ F_{m}$. Specifically, $\mathrm{ACC}$ and $\mathrm{NMI}$ are used to evaluate the cluster validity. The metrics $\mathrm{ Bal}$ and $\mathrm{MNCE}$  are used to evaluate the fairness performance. Finally, $F_m$ is a harmonic mean of $\mathrm{NMI}$ and $\mathrm{MNCE}$ that considers both clustering and fairness performance. The detailed definitions of these metrics are placed in the appendix. For the above metrics, higher values represent better results. A natural metric for assessing the fairness of clustering results in the continuous case is $I(C;G)$, as suggested by Definition \ref{fairness-unified}. However, it is difficult to exactly estimate mutual information, especially for continuous variables. Thus, in this study, we follow \cite{chen2022scalable} and use R\'{e}nyi maximal correlation $\rho^*(\cdot)$ 
to estimate statistical dependence as a proxy of mutual information, the computation of which is presented in the appendix. The value range of  $\rho^*(\cdot)$ is $[0,1]$, and  $\rho^*(C,G)=0$ if and only if $C\perp G$, meaning that $C$ is insensitive to $G$. Thus, smaller $\rho^*(\cdot)$ means fairer results. For clarity, the results of all metrics are presented in the form of percentages.

\subsection{ANSW 1: Performance on Discrete Sensitive Attributes}
We first report quantitative comparisons between our method and ten baselines in Table \ref{table-discrete}. It is observed that the traditional clustering methods perform poorly on the fairness metrics since they do not consider fairness constraints. Note that although ScFc achieves the highest fairness metric across multiple datasets, it does so at the expense of clustering quality. In contrast, deep fair clustering methods can achieve improved fairness while maintaining high-quality clustering. Compared with two deep methods, DFC and FCMI, our method achieves superior or competitive performance on most datasets except Office-31. However, the results of our method on Office-31 still outperform traditional methods. 

Figure \ref{fig-visual} depicts the t-SNE visualization results \cite{van2008visualizing}  of the representations learned by our method. At the initial epoch, all representations are entangled with each other. As training goes on, the representations show compact clusters. Besides, the sensitive attributes are also evenly distributed across clusters, bringing fair clustering results. Figure \ref{param-sensitive} describes the parameter sensitivity of our method. It is observed that the best performance of our method is around $\alpha = 0.04$ and $\beta = 0.2$. Moreover, for different $\alpha$ and $\beta$, our method achieves stable performance, indicating its robustness.

\begin{figure}[t] 
    \centering
\subfloat[$\mathrm{ACC}$]{
    \includegraphics[width=0.49\linewidth]{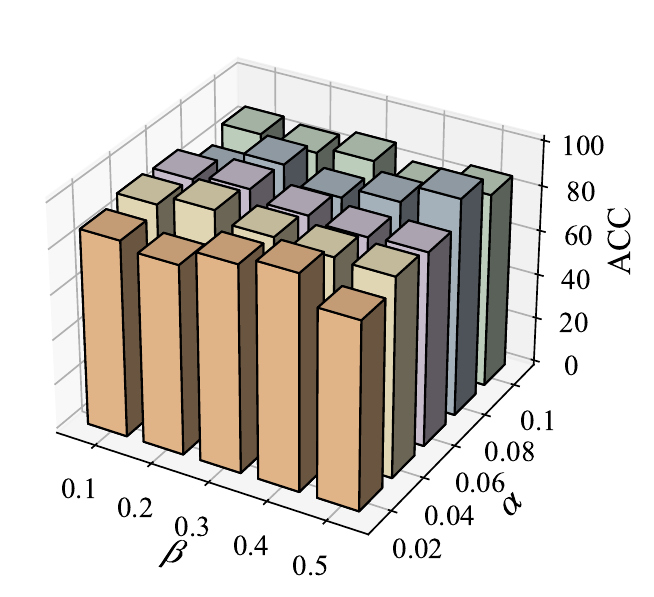}
}
\subfloat[$\mathrm{MNCE}$]{
    \centering
    \includegraphics[width=0.49\linewidth]{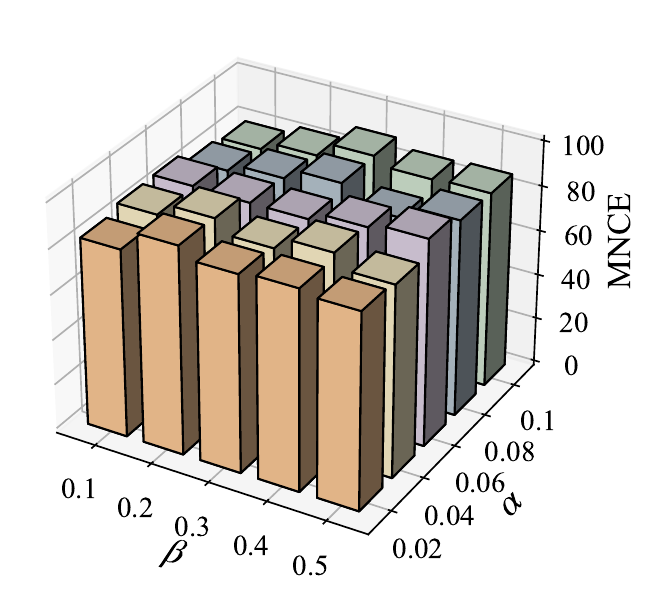}
}

    	\caption{The effect of $\alpha$ and $\beta$ on $\mathrm{ACC}$ and $\mathrm{MNCE}$ of MNIST-USPS dataset
     }
    	\label{param-sensitive}
\end{figure}

\subsection{ANSW 2: Performance on Continuous Sensitive Attributes}
For all datasets, we add sensitive variables into input features to better illustrate the ability of all methods to remove sensitive information. As displayed in Table \ref{table-continuous}, the values of $\rho^*(C,G)$ of baselines are higher than ours in both datasets, indicating that our method can obtain fairer clustering results. The reason may be that discretized sensitive variables cannot fully reflect sensitive information of continuous attributes. On the contrary, our method removes statistical dependence between $Z$ and $G$ directly without discretization, which also leads to decreasing dependence between $C$ and $G$.


To better illustrate the effect of each module, we conduct two ablation studies on the Census dataset, namely removing $\mathcal{L}_c$ (w/o $\mathcal{L}_c$) and $\mathcal{L}_s$ (w/o $\mathcal{L}_s$). As shown in Fig.\ref{fig-ablation}, when  $\mathcal{L}_c$ is removed, metrics related to clustering significantly decrease.  On the other hand, when $\mathcal{L}_s$ is removed, the statistical dependence, $\rho^*(Z,G)$ and $\rho^*(C,G)$, suffer from significant increase. Thus, the ablation studies verify the contribution of the proposed objective function to learning compact, balanced, and fair representations for clustering tasks.

\begin{table}[t]
\renewcommand{\arraystretch}{1.5}
	\centering
	\begin{threeparttable}
        \tabcolsep = 0.4em
	\caption{Clustering results on the datasets with continuous sensitive attributes.}
	{\footnotesize
	\begin{tabular}
   {c|ccc|ccc}
	\Xhline{1.05pt}

    	 & \multicolumn{3}{c|}{Census} & \multicolumn{3}{c}{Crime}\\

	\cline{2-7}

     & $\mathrm{ACC} $ & $\mathrm{NMI} $ 
     & $\rho^*(C, G)$ & $\mathrm{ACC} $ & $\mathrm{NMI} $  
     & $\rho^*(C, G)$  \\
     
    \hline
    DFC 
    & 67.3 & 10.2 
    & 9.8 
    & 62.2 & 9.6 
    & 35.3 \\

    FCMI
    & 70.6 & 12.8 
    & 4.8 
    & 67.0 & 9.1 
    & 33.7 \\

    Ours
    & \textbf{79.1} & \textbf{26.0} 
    & \textbf{0.4}
    & \textbf{69.3} & \textbf{12.1} 
    & \textbf{10.5} \\

 \Xhline{1.05pt}
	\end{tabular} 
	}

          \label{table-continuous}
 \end{threeparttable}
\end{table}

\begin{figure}[t] 
    \centering
\subfloat[Clustering-related metrics]{
    \includegraphics[width=0.8\linewidth]{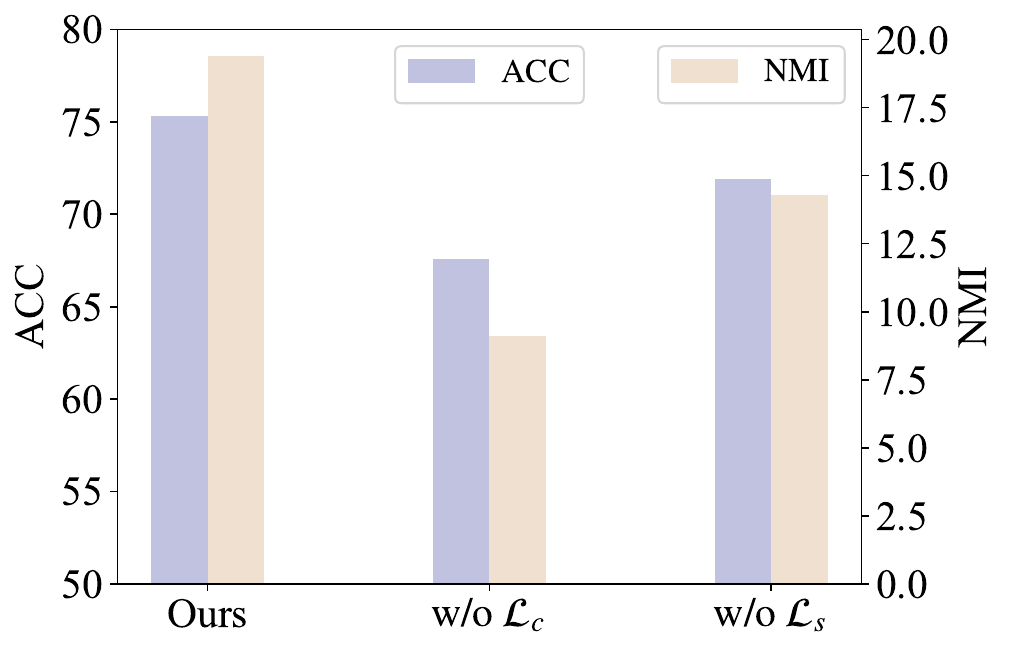}
}

\subfloat[Fairness-related metrics]{
    \centering
    \includegraphics[width=0.8\linewidth]{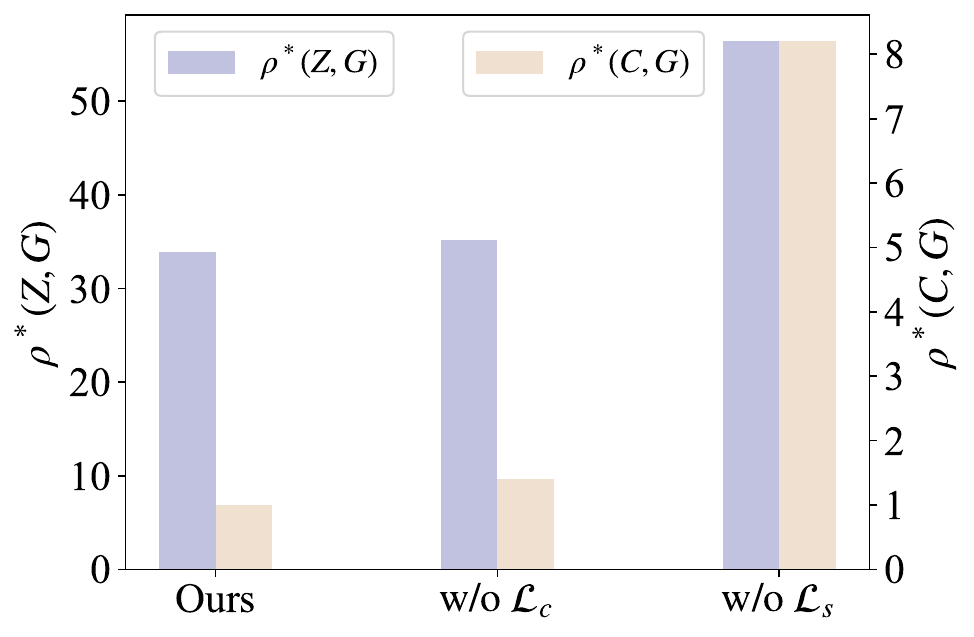}
}

    	\caption{Ablation study on the Census dataset
     }
    	\label{fig-ablation}
\end{figure}

\subsection{ANSW3: Performance on Transferable Representations }
Before we transfer the learned representations to other downstream tasks, we first test the fairness of the representations learned by different deep methods. Specifically, we calculate the statistical dependence between representations and the sensitive variables, $\rho^*(Z,G)$, for both MNIST-USPS and Census datasets. As described in Fig.\ref{fig-representation-fairness}, AE presents the worst performance since it does not consider fairness. The representations of FCMI and DFC have higher $\rho^*(Z,G)$ because they learn representations by minimizing $I(C;G)$, which may not necessarily decrease the statistical dependence of $Z$ w.r.t. $G$. FRL is the model in \cite{madras2018learning} that learns fair and transferable representations. Similar to ours, FRL also directly removes sensitive information at the presentation level. The method is inferior to ours, especially for the continuous case, since they can only handle binary-sensitive attributes. The results support our motivation to directly minimize $I(Z;G)$.

Then, we transfer the learned representations to classification tasks. We focus on classification because the representations learned from the clustering task, which is unsupervised,  exhibit cluster structures as displayed in Fig.\ref{fig-visual}. Therefore, these representations could facilitate classification tasks, especially when only a few labels are available. Specifically, we freeze the encoder learned from the fair clustering task and use it as a representation generator. We then use the representations as inputs to train a fully connected classifier with the same structure as the network of estimating $I(Z;G)$. For comparison, we take representations generated by (\romannumeral1) DFC, (\romannumeral2) FCMI, (\romannumeral3) auto-encoder (AE), and (\romannumeral4) FRL without labels \cite{madras2018learning} \footnote{We remove the label-related term in the corresponding objective function and use binarized sensitive attributes in the continuous case} as inputs to train the classifier. We employ two datasets, one with discrete sensitive variables (MNIST-USPS) and the other with continuous sensitive variables (Census). For each dataset, we only sample 128 data as training data and the rest as testing data. The labels of the two datasets are digits and whether the average income of a tract is greater than 50,000 USD, respectively. We use accuracy to evaluate classification performance. Furthermore, we use generalized demographic parity $\Delta_{GDP}$ \eqref{definition-SP} and $\rho^*(\widehat{Y},G)$ as the fairness metrics of discrete and continuous sensitive attributes, respectively. As shown in Fig.\ref{fig-transfer}, the representations learned from clustering tasks achieve higher classification accuracy than those without clustering objective functions (AE and FRL). Our method achieves the highest classification accuracy, which is close to 1 although we only have 128 training data. The reason may be that the clustering objectives in our method help generate representations with compact clusters, which may boost classification tasks. On the other hand, our method also obtains the lowest $\Delta_{GDP}$, meaning that the classification results based on the representations of our method are fairer. This can be explained by that we minimize $I(Z;G)$ directly, which also bounds $\Delta_{GDP}$ as Proposition \ref{proposition-SP} states. For the Census data, our method achieves the lowest $\rho^*(\widehat{Y},G)$ while retaining utility as much as possible. The other methods are inferior to ours since they cannot handle continuous sensitive variables. In summary, the clustering objectives can produce representations that can boost few-shot classification tasks. Furthermore, the fairness of downstream tasks can also be guaranteed since we minimize $I(Z;G)$ instead of $I(C;G)$.

\begin{figure}[t] 
    \centering
    \includegraphics[width=0.75\linewidth]{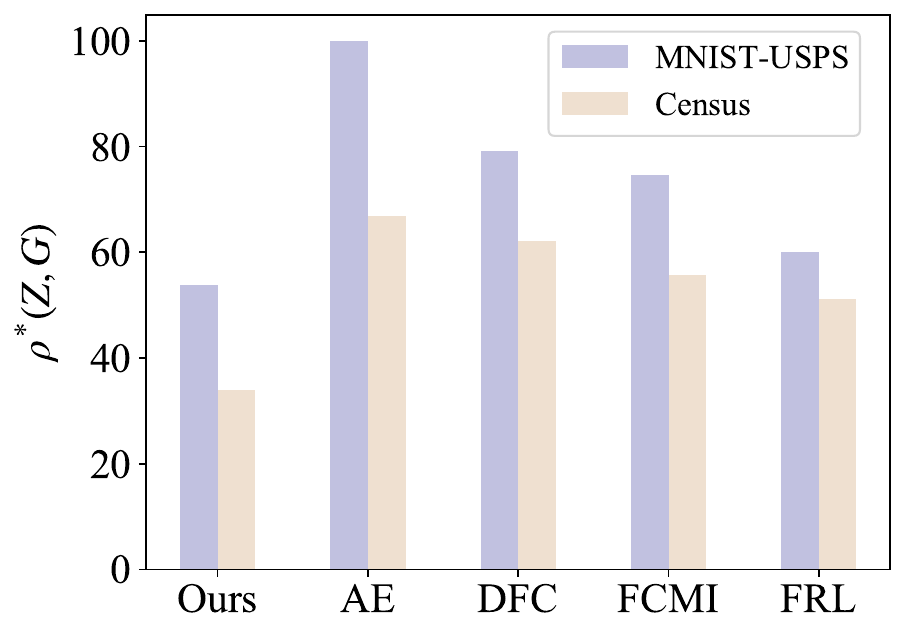}
    	\caption{The fairness of the learned representations.
     }
    	\label{fig-representation-fairness}
\end{figure}

\begin{figure}[t] 
    \centering
\subfloat[MNIST-USPS]{
    \includegraphics[width=0.8\linewidth]{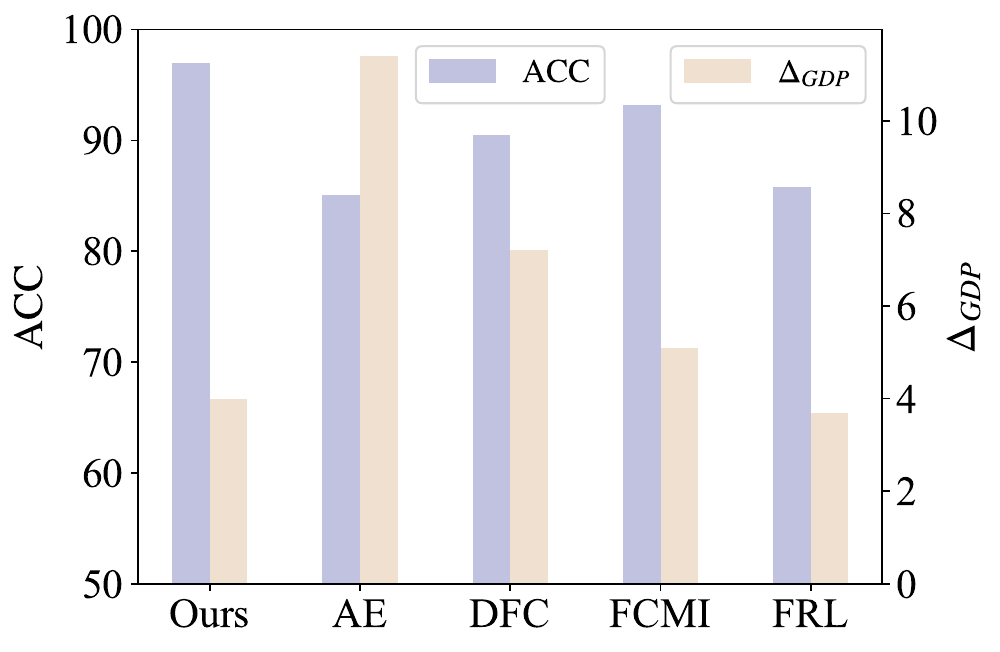}
}

\subfloat[Census]{
    \centering
    \includegraphics[width=0.8\linewidth]{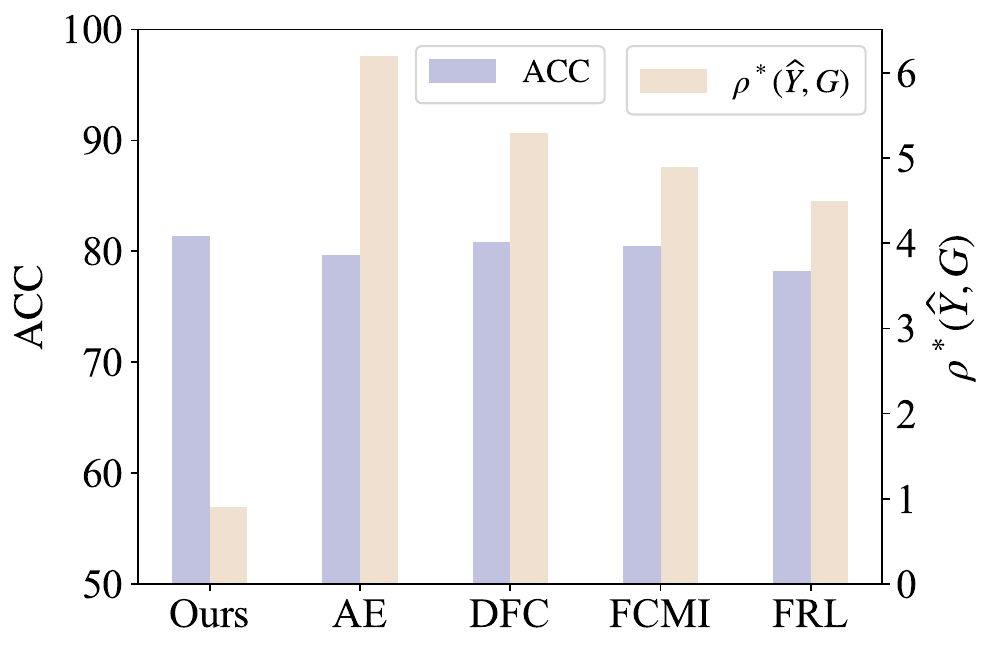}
}

    	\caption{The performance of transferring to classification tasks.
     }
    	\label{fig-transfer}
\end{figure}

%% file: Conclusion/conclusion.tex
In this paper, we proposed a deep fair clustering method that can handle both discrete and continuous sensitive attributes. Furthermore, we investigated the transferability of representations learned from clustering tasks to unseen tasks. Extensive experiments, including discrete and continuous sensitive attributes, were conducted. The results showed that our method can achieve superior performance in terms of both clustering and fairness performance.

%% file: Supplementary/supplementary.tex
\subsection{Proof of Proposition \ref{proposition-SP}}
The proof of this proposition mainly comes from \cite{gupta2021controllable}, but with some modifications to adapt to the definition of generalized DP. First, we provide the following lemma.
\begin{lemma}
\begin{align}
       I(Z;G)\geq h\left(\eta\max_{t,t^{\prime}}\left \lVert p(Z|G = t) - p(Z|G = t^{\prime}) \right\rVert_v \right ), \label{IZG}
\end{align}
where $\eta = \min_{t\in[T]} p(G = t)$, $\lVert\cdot\rVert_v$ denotes the variational distance between two probabilities, and $h(v)$ is
\begin{align}
h(v) = \max\left(\log(\frac{2+v}{2-v}) - \frac{2v}{2+v}, \frac{v^2}{2} + \frac{v^4}{36}  +  \frac{v^6}{288} \right).
\end{align}
    \label{lemma-1}
\end{lemma}
The proof of Lemma \ref{lemma-1} is in \cite{gupta2021controllable}. Then, by the definition of GDP, we have 
\begin{align}
&\Delta_{GDP} (\lambda, G) \notag\\
= &\max_{l, t\neq t^{\prime}}\left|p( \widehat{Y} = l |G = t) - p( \widehat{Y} = l |G = t^{\prime}) \right|\label{gdp-1}\\
= & \max_{l, t\neq t^{\prime}}\Bigg{|} \int_{Z} dZ p( \widehat{Y} = l |Z) p(Z|G =  t)\notag\\
&\;\;\;\;\;\;\;\;\;\;-\int_{Z} dZ p( \widehat{Y} = l |Z) p(Z|G =  t^{\prime})\Bigg{|}\notag\\
=&\Bigg{|} \int_{Z} dZ p( \widehat{Y} = l^* |Z) p(Z|G =  t^*)\notag\\
&\;\;\;\;\;\;\;\;\;\;-\int_{Z} dZ p( \widehat{Y} = l^* |Z) p(Z|G =  t^{\prime*})\Bigg{|}\notag\\
\leq&\int_{Z} dZ p( \widehat{Y} = l^* |Z) \left|p(Z|G =  t^*) - p(Z|G =  t^{\prime*})\right|\notag\\
\leq& \int_{Z} dZ \left|p(Z|G =  t^*) - p(Z|G =  t^{\prime*})\right| \notag\\
= & \left \lVert p(Z|G = t^*) - p(Z|G = t^{\prime*}) \right\rVert_v, \label{gdp}
\end{align}
where $t^*, l^* $ and $t^{\prime*}$ are those reach the maximum of \eqref{gdp-1}. Furthermore, the last inequality holds since $ p( \widehat{Y} = l^* |Z) \leq1$. Bringing \eqref{gdp} back to \eqref{IZG}, we have that 
\begin{align}
    h(\eta \Delta_{GDP} (\lambda, G)) \leq I(Z;G),
    \label{gdp-2}
\end{align}
where $h$ is strictly increasing, non-negative, and convex. Finally, we complete the proof.

\subsection{Definitions of Evaluation Metrics}
The definitions of $\mathrm{ACC}$ and $\mathrm{NMI}$ are as follows: 
\begin{align}
\mathrm{ACC} &= \sum_{i}^N\frac{\mathbb{I}(\mathrm{map}(C_i) = Y_i)}{N}, \label{acc}\\
\mathrm{NMI} &= \frac{\sum_{j,k} N_{jk}\log\frac{N N_{jk}}{N_{j+} N_{+k}}}{\sqrt{\left( \sum_{j} N_{j+}\log \frac{N_{j+}}{N} \right)\left( \sum_{k} N_{+k}\log \frac{N_{+k}}{N} \right)}} \label{nmi},
\end{align}
where $\mathrm{map}(C_i)$ is a permutation mapping function that maps each cluster label $C_i$ to the ground truth label $Y_i$, $N_{jk}, N_{j+}$ and $N_{+k}$ represent the co-occurrence number and cluster size of $j$-th and $k$-th clusters in the obtained partition and ground truth, respectively. The definitions of $\mathrm{Bal}$ and $\mathrm{MNCE}$ are: 
\begin{align}
\mathrm{Bal} &= \min_{k} \frac{|\mathcal{C}_k \cap \mathcal{G}_t |}{N_{k+}},\label{bal}\\
\mathrm{MNCE} &= \frac{\min_k\left(-\sum_t^T \frac{|\mathcal{G}_t \cap \mathcal{C}_k|}{|\mathcal{C}_k|} \log \frac{|\mathcal{G}_t \cap \mathcal{C}_k|}{|\mathcal{C}_k|}\right)}{-\sum_t^T \frac{|\mathcal{G}_t|}{N}\log \frac{|\mathcal{G}_t|}{N}}. \label{MNCE}
\end{align}
Finally, $F_m$ is a harmonic mean of $\mathrm{MNI}$ and $\mathrm{MNCE}$
\begin{align}
F_m = \frac{(1 + m^2) \mathrm{MNI}\cdot \mathrm{MNCE}}{m^2\mathrm{MNI} + \mathrm{MNCE}},\label{Fm}
\end{align}
where $m$ in a constant. For two variables, say $C$ and $G$,  R\'{e}nyi maximal correlation $\rho^*(C,G)$  is estimated as 
\begin{align}
\rho^*(C,G) = \sup_{a, b} \rho(a(C), b(G)),
    \label{equ-expe-1}
\end{align}
where $\rho$ is the Pearson correlation. Following \cite{chen2022scalable}, we approximate \eqref{equ-expe-1} by two neural networks $a$ and $b$. The reliability of this neural approximation has been
verified in \cite{grari2021learning}.

\subsection{Details of Used Datasets}
We list the details of the used dataset in Table \ref{table-datsets}

\begin{table}[h]
\renewcommand{\arraystretch}{1.5}
	\centering
	\begin{threeparttable}
        \tabcolsep = 0.1em
    	\caption{Details of the used data.}
	{\footnotesize
	\begin{tabular}%
 {c|cccc}
	\Xhline{1.05pt}

    Dataset & \#Samples & \#Clusters&   Cluster Label \\
     
    \Xhline{1.05pt}
    MNIST-USPS 
    & 67,291 & 10  & Digit \\

    Reverse-MNIST
    & 120,000 & 10  & Digit \\

    HAR
    & 10,299 & 6  & Activity \\
    
    Office-31
    & 3,612 & 31 & Item Category \\

    MTFL
    & 2,000 & 2  & Gender \\

    Census
    & 74,000 & 2  & Income$\geq$ 50,000 USD \\

    Crime 
    &  1,994 & 2  &Crime rate$\geq$ 0.15 \\
\Xhline{1.05pt}
  Dataset& \#Groups& 
  Sensitive Attributes  & Sensitive Type \\
  \Xhline{1.05pt}
  MNIST-USPS& 2  & Domain Source & Discrete \\
  ReverseMNIST& 2  & Background Color &  Discrete \\
  HAR& 30 &  Subject &  Discrete \\
  Office-31& 2  &  Domain Source & Discrete\\
  MTFL& 2  &  w/ or w/o Glass & Discrete\\
  Census& \textemdash & Population proportion of female & Continuous\\
  Crime& \textemdash &   Population proportion of a  race& Continuous\\

 \Xhline{1.05pt}
	\end{tabular} 
	}

          \label{table-datsets}
 \end{threeparttable}
     	\vspace{-1em}
\end{table}

%% file: main.bbl
\begin{thebibliography}{10}

\bibitem{lei2018superpixel}
T.~Lei, X.~Jia, Y.~Zhang, S.~Liu, H.~Meng, and A.~K. Nandi, ``Superpixel-based
  fast fuzzy c-means clustering for color image segmentation,'' {\em {IEEE}
  Trans. Fuzzy Syst.}, vol.~27, no.~9, pp.~1753--1766, 2018.

\bibitem{xie2018unsupervised}
H.~Xie, A.~Zhao, S.~Huang, J.~Han, S.~Liu, X.~Xu, X.~Luo, H.~Pan, Q.~Du, and
  X.~Tong, ``Unsupervised hyperspectral remote sensing image clustering based
  on adaptive density,'' {\em {IEEE} Geosci. Remote S.}, vol.~15, no.~4,
  pp.~632--636, 2018.

\bibitem{kiselev2019challenges}
V.~Y. Kiselev, T.~S. Andrews, and M.~Hemberg, ``Challenges in unsupervised
  clustering of single-cell rna-seq data,'' {\em Nat. Rev. Genet.}, vol.~20,
  no.~5, pp.~273--282, 2019.

\bibitem{chouldechova2018frontiers}
A.~Chouldechova and A.~Roth, ``The frontiers of fairness in machine learning,''
  {\em arXiv:1810.08810}, 2018.

\bibitem{Zeng_2023_CVPR}
P.~Zeng, Y.~Li, P.~Hu, D.~Peng, J.~Lv, and X.~Peng, ``Deep fair clustering via
  maximizing and minimizing mutual information: Theory, algorithm and metric,''
  in {\em {Proc.} {IEEE} Conf. Comput. Vis. Pattern Recognit.},
  pp.~23986--23995, June 2023.

\bibitem{chierichetti2017fair}
F.~Chierichetti, R.~Kumar, S.~Lattanzi, and S.~Vassilvitskii, ``Fair clustering
  through fairlets,'' {\em Proc. Adv. Neural Inf. Process. Syst.}, vol.~30,
  2017.

\bibitem{backurs2019scalable}
A.~Backurs, P.~Indyk, K.~Onak, B.~Schieber, A.~Vakilian, and T.~Wagner,
  ``Scalable fair clustering,'' in {\em Proc. Int. Conf. Mach. Learn.},
  pp.~405--413, PMLR, 2019.

\bibitem{kleindessner2019guarantees}
M.~Kleindessner, S.~Samadi, P.~Awasthi, and J.~Morgenstern, ``Guarantees for
  spectral clustering with fairness constraints,'' in {\em Proc. Int. Conf.
  Mach. Learn.}, pp.~3458--3467, PMLR, 2019.

\bibitem{ziko2021variational}
I.~M. Ziko, J.~Yuan, E.~Granger, and I.~B. Ayed, ``Variational fair
  clustering,'' in {\em Proc. Natl. Conf. Artif. Intell.}, vol.~35,
  pp.~11202--11209, 2021.

\bibitem{bera2019fair}
S.~Bera, D.~Chakrabarty, N.~Flores, and M.~Negahbani, ``Fair algorithms for
  clustering,'' {\em Proc. Adv. Neural Inf. Process. Syst.}, vol.~32, 2019.

\bibitem{wang2019towards}
B.~Wang and I.~Davidson, ``Towards fair deep clustering with multi-state
  protected variables,'' {\em arXiv preprint arXiv:1901.10053}, 2019.

\bibitem{zhang2021deep}
H.~Zhang and I.~Davidson, ``Deep fair discriminative clustering,'' {\em arXiv
  preprint arXiv:2105.14146}, 2021.

\bibitem{li2020deep}
P.~Li, H.~Zhao, and H.~Liu, ``Deep fair clustering for visual learning,'' in
  {\em {Proc.} {IEEE} Conf. Comput. Vis. Pattern Recognit.}, pp.~9070--9079,
  2020.

\bibitem{grari2019fairness}
V.~Grari, B.~Ruf, S.~Lamprier, and M.~Detyniecki, ``Fairness-aware neural
  r\'{e}yni minimization for continuous features,'' {\em arXiv preprint
  arXiv:1911.04929}, 2019.

\bibitem{mary2019fairness}
J.~Mary, C.~Calauzenes, and N.~El~Karoui, ``Fairness-aware learning for
  continuous attributes and treatments,'' in {\em Proc. Int. Conf. Mach.
  Learn.}, pp.~4382--4391, PMLR, 2019.

\bibitem{madras2018learning}
D.~Madras, E.~Creager, T.~Pitassi, and R.~Zemel, ``Learning adversarially fair
  and transferable representations,'' in {\em Proc. Int. Conf. Mach. Learn.},
  pp.~3384--3393, PMLR, 2018.

\bibitem{ahmadian2019clustering}
S.~Ahmadian, A.~Epasto, R.~Kumar, and M.~Mahdian, ``Clustering without
  over-representation,'' in {\em Proc. ACM SIGKDD Int. Conf. Knowl. Discov.
  Data Min.}, pp.~267--275, 2019.

\bibitem{brubach2020pairwise}
B.~Brubach, D.~Chakrabarti, J.~Dickerson, S.~Khuller, A.~Srinivasan, and
  L.~Tsepenekas, ``A pairwise fair and community-preserving approach to
  k-center clustering,'' in {\em Proc. Int. Conf. Mach. Learn.},
  pp.~1178--1189, PMLR, 2020.

\bibitem{chen2019proportionally}
X.~Chen, B.~Fain, L.~Lyu, and K.~Munagala, ``Proportionally fair clustering,''
  in {\em Proc. Int. Conf. Mach. Learn.}, pp.~1032--1041, PMLR, 2019.

\bibitem{davidson2020making}
I.~Davidson and S.~Ravi, ``Making existing clusterings fairer: Algorithms,
  complexity results and insights,'' in {\em Proc. Natl. Conf. Artif. Intell.},
  vol.~34, pp.~3733--3740, 2020.

\bibitem{mahabadi2020individual}
S.~Mahabadi and A.~Vakilian, ``Individual fairness for k-clustering,'' in {\em
  Proc. Int. Conf. Mach. Learn.}, pp.~6586--6596, PMLR, 2020.

\bibitem{ghasedi2017deep}
K.~Ghasedi~Dizaji, A.~Herandi, C.~Deng, W.~Cai, and H.~Huang, ``Deep clustering
  via joint convolutional autoencoder embedding and relative entropy
  minimization,'' in {\em Int. Conf. Comput. Vis.}, pp.~5736--5745, 2017.

\bibitem{guo2017deep}
X.~Guo, X.~Liu, E.~Zhu, and J.~Yin, ``Deep clustering with convolutional
  autoencoders,'' in {\em in Proc. Int. Conf. Neural Inf. Process.},
  pp.~373--382, Springer, 2017.

\bibitem{li2021contrastive}
Y.~Li, P.~Hu, Z.~Liu, D.~Peng, J.~T. Zhou, and X.~Peng, ``Contrastive
  clustering,'' in {\em Proc. Natl. Conf. Artif. Intell.}, vol.~35,
  pp.~8547--8555, 2021.

\bibitem{vincent2010stacked}
P.~Vincent, H.~Larochelle, I.~Lajoie, Y.~Bengio, P.-A. Manzagol, and L.~Bottou,
  ``Stacked denoising autoencoders: Learning useful representations in a deep
  network with a local denoising criterion.,'' {\em J. Mach. Learn. Res.},
  vol.~11, no.~12, 2010.

\bibitem{yang2019deep}
X.~Yang, C.~Deng, F.~Zheng, J.~Yan, and W.~Liu, ``Deep spectral clustering
  using dual autoencoder network,'' in {\em {Proc.} {IEEE} Conf. Comput. Vis.
  Pattern Recognit.}, pp.~4066--4075, 2019.

\bibitem{chhabra2022robust}
A.~Chhabra, P.~Li, P.~Mohapatra, and H.~Liu, ``Robust fair clustering: A novel
  fairness attack and defense framework,'' {\em arXiv preprint
  arXiv:2210.01953}, 2022.

\bibitem{lee2022maximal}
J.~Lee, Y.~Bu, P.~Sattigeri, R.~Panda, G.~Wornell, L.~Karlinsky, and R.~Feris,
  ``A maximal correlation approach to imposing fairness in machine learning,''
  in {\em Proc. IEEE Int. Conf. Acoust., Speech, Signal Process.},
  pp.~3523--3527, IEEE, 2022.

\bibitem{chen2022scalable}
Y.~Chen, Y.~Li, A.~Weller, {\em et~al.}, ``Scalable infomin learning,'' {\em
  Proc. Adv. Neural Inf. Process. Syst.}, vol.~35, pp.~2226--2239, 2022.

\bibitem{zemel2013learning}
R.~Zemel, Y.~Wu, K.~Swersky, T.~Pitassi, and C.~Dwork, ``Learning fair
  representations,'' in {\em Proc. Int. Conf. Mach. Learn.}, pp.~325--333,
  PMLR, 2013.

\bibitem{zhu2021learning}
W.~Zhu, H.~Zheng, H.~Liao, W.~Li, and J.~Luo, ``Learning bias-invariant
  representation by cross-sample mutual information minimization,'' in {\em
  {Proc.} {IEEE} Conf. Comput. Vis. Pattern Recognit.}, pp.~15002--15012, 2021.

\bibitem{shui2022fair}
C.~Shui, Q.~Chen, J.~Li, B.~Wang, and C.~Gagn{\'e}, ``Fair representation
  learning through implicit path alignment,'' in {\em Proc. Int. Conf. Mach.
  Learn.}, pp.~20156--20175, PMLR, 2022.

\bibitem{gordaliza2019obtaining}
P.~Gordaliza, E.~Del~Barrio, G.~Fabrice, and J.-M. Loubes, ``Obtaining fairness
  using optimal transport theory,'' in {\em Proc. Int. Conf. Mach. Learn.},
  pp.~2357--2365, PMLR, 2019.

\bibitem{8438994}
F.~d.~P. Calmon, D.~Wei, B.~Vinzamuri, K.~N. Ramamurthy, and K.~R. Varshney,
  ``Data pre-processing for discrimination prevention: Information-theoretic
  optimization and analysis,'' {\em {IEEE} J. Sel. Topics Signal Process.},
  vol.~12, no.~5, pp.~1106--1119, 2018.

\bibitem{mehrabi2021survey}
N.~Mehrabi, F.~Morstatter, N.~Saxena, K.~Lerman, and A.~Galstyan, ``A survey on
  bias and fairness in machine learning,'' {\em ACM Comput. Surv.}, vol.~54,
  no.~6, pp.~1--35, 2021.

\bibitem{gupta2021controllable}
U.~Gupta, A.~M. Ferber, B.~Dilkina, and G.~Ver~Steeg, ``Controllable guarantees
  for fair outcomes via contrastive information estimation,'' in {\em Proc.
  Natl. Conf. Artif. Intell.}, vol.~35, pp.~7610--7619, 2021.

\bibitem{song2019learning}
J.~Song, P.~Kalluri, A.~Grover, S.~Zhao, and S.~Ermon, ``Learning controllable
  fair representations,'' in {\em Proc. Int. Conf. Artif. Intell. Stat.,
  AISTATS}, pp.~2164--2173, PMLR, 2019.

\bibitem{xie2017controllable}
Q.~Xie, Z.~Dai, Y.~Du, E.~Hovy, and G.~Neubig, ``Controllable invariance
  through adversarial feature learning,'' {\em Proc. Adv. Neural Inf. Process.
  Syst.}, vol.~30, 2017.

\bibitem{barocas2023fairness}
S.~Barocas, M.~Hardt, and A.~Narayanan, {\em Fairness and machine learning:
  Limitations and opportunities}.
\newblock MIT Press, 2023.

\bibitem{dwork2012fairness}
C.~Dwork, M.~Hardt, T.~Pitassi, O.~Reingold, and R.~Zemel, ``Fairness through
  awareness,'' in {\em Proc. 3rd Innovations Theoretical Comput. Sci. Conf.},
  pp.~214--226, 2012.

\bibitem{ren2022deep}
Y.~Ren, J.~Pu, Z.~Yang, J.~Xu, G.~Li, X.~Pu, P.~S. Yu, and L.~He, ``Deep
  clustering: A comprehensive survey,'' {\em arXiv preprint arXiv:2210.04142},
  2022.

\bibitem{cheng2020club}
P.~Cheng, W.~Hao, S.~Dai, J.~Liu, Z.~Gan, and L.~Carin, ``Club: A contrastive
  log-ratio upper bound of mutual information,'' in {\em Proc. Int. Conf. Mach.
  Learn.}, pp.~1779--1788, PMLR, 2020.

\bibitem{doersch2016tutorial}
C.~Doersch, ``Tutorial on variational autoencoders,'' {\em arXiv preprint
  arXiv:1606.05908}, 2016.

\bibitem{anguita2013public}
D.~Anguita, A.~Ghio, L.~Oneto, X.~Parra, J.~L. Reyes-Ortiz, {\em et~al.}, ``A
  public domain dataset for human activity recognition using smartphones.,'' in
  {\em Proc. ESANN}, vol.~3, p.~3, 2013.

\bibitem{mukherjee2019clustergan}
S.~Mukherjee, H.~Asnani, E.~Lin, and S.~Kannan, ``Clustergan: Latent space
  clustering in generative adversarial networks,'' in {\em Proc. Natl. Conf.
  Artif. Intell.}, vol.~33, pp.~4610--4617, 2019.

\bibitem{zhang2014facial}
Z.~Zhang, P.~Luo, C.~C. Loy, and X.~Tang, ``Facial landmark detection by deep
  multi-task learning,'' in {\em Proc. Eur. Conf. Comput. Vis.}, pp.~94--108,
  Springer, 2014.

\bibitem{he2016deep}
K.~He, X.~Zhang, S.~Ren, and J.~Sun, ``Deep residual learning for image
  recognition,'' in {\em {Proc.} {IEEE} Conf. Comput. Vis. Pattern Recognit.},
  pp.~770--778, 2016.

\bibitem{xie2016unsupervised}
J.~Xie, R.~Girshick, and A.~Farhadi, ``Unsupervised deep embedding for
  clustering analysis,'' in {\em Proc. Int. Conf. Mach. Learn.}, pp.~478--487,
  PMLR, 2016.

\bibitem{chang2017deep}
J.~Chang, L.~Wang, G.~Meng, S.~Xiang, and C.~Pan, ``Deep adaptive image
  clustering,'' in {\em Int. Conf. Comput. Vis.}, pp.~5879--5887, 2017.

\bibitem{van2008visualizing}
L.~Van~der Maaten and G.~Hinton, ``Visualizing data using t-sne.,'' {\em J.
  Mach. Learn. Res.}, vol.~9, no.~11, 2008.

\bibitem{grari2021learning}
V.~Grari, O.~E. Hajouji, S.~Lamprier, and M.~Detyniecki, ``Learning unbiased
  representations via r{\'e}nyi minimization,'' in {\em Joint Eur. Conf. Mach.
  Learn. Knowl. Discov. Database}, pp.~749--764, Springer, 2021.

\end{thebibliography}
